\documentclass{article}
\usepackage{microtype}
\usepackage{graphicx}
\usepackage{subfig}
\usepackage{booktabs} 
\usepackage{amsmath, amsthm, amssymb,xspace}
\usepackage{hyperref}
\usepackage{bm}
\usepackage{tikz}
\usepackage{comment}
\usepackage{mathtools}

\newtheorem{prop}{Proposition}
\newtheorem{defn}{Definition}
\newtheorem{theorem}{Theorem}
\newtheorem{lemma}{Lemma}
\newtheorem*{notation*}{Notation}

\usepackage[accepted]{icml2018}

\makeatletter
\DeclareRobustCommand\onedot{\futurelet\@let@token\@onedot}
\def\@onedot{\ifx\@let@token.\else.\null\fi\xspace}

\def\eg{\emph{e.g}\onedot}

\makeatother

\icmltitlerunning{Block Mean Approximation for Efficient Second Order Optimization}

\begin{document}

\twocolumn[
\icmltitle{Block Mean Approximation for Efficient Second Order Optimization}

\begin{icmlauthorlist}
	\icmlauthor{Yao Lu}{anu,d61}
	\icmlauthor{Mehrtash Harandi}{anu,d61}
    \icmlauthor{Richard Hartley}{anu}
    \icmlauthor{Razvan Pascanu}{dp}
\end{icmlauthorlist}

\icmlaffiliation{d61}{Data61, CSIRO}
\icmlaffiliation{anu}{Australian National University}
\icmlaffiliation{dp}{Google DeepMind}

\icmlcorrespondingauthor{\\ Yao Lu}{yaolubrain@gmail.com}

\icmlkeywords{Deep Neural Networks, Natural Gradient, Optimization}

\vskip 0.3in
]

\printAffiliationsAndNotice{}

\begin{abstract}

Advanced optimization algorithms such as Newton method and AdaGrad 
benefit from second order derivative or second order statistics to achieve better descent directions and faster convergence rates.
At their heart, such algorithms need to compute the inverse or inverse square root of a matrix whose size is quadratic of the dimensionality of the search space. For high dimensional search spaces, the matrix inversion or inversion of square root becomes overwhelming which in turn demands for approximate methods. 
In this work, we propose a new matrix approximation method which divides a matrix into blocks and represents each block by one or two numbers. The method allows efficient computation of matrix inverse and inverse square root. We apply our method to AdaGrad in training deep neural networks. Experiments show encouraging results compared to the diagonal approximation.
\end{abstract}

\section{Introduction}

Gradient-based optimization algorithms usually have the following update form,
\begin{align}
\bm{\theta} \leftarrow  \bm{\theta} - \eta \mathbf{G}^{-1} \nabla_{\bm{\theta}}f(\bm{\theta})
\label{gradient_descent}
\end{align}
where $\bm{\theta} \in \mathbb{R}^d$ is the model parameters, $\eta$ is a learning rate, $\mathbf{G}$ is a $d\times d$ non-singular matrix and $f(\cdot)$ is the loss function. If $\mathbf{G}$ is the identity matrix, (\ref{gradient_descent}) is the gradient descent method. 

In practice, gradient descent often converges slowly and its performance depends critically on how $f(\bm{\theta})$ is parameterized. That is, 
if one chooses $\bm{\theta} = g(\bm{\xi})$, 
then the behavior of $\bm{\theta} \leftarrow  \bm{\theta} - \eta \nabla_{\bm{\theta}}f(\bm{\theta})$ can be significantly different from that of $\bm{\xi} \leftarrow  \bm{\xi} - \eta \nabla_{\bm{\xi}}f(g(\bm{\xi}))$. (see an example in~\textsection~\ref{sec:parameter_dependency}).

\newpage

To accelerate gradient descent and deal better with the parameterization issue, one can resort to second order optimization methods, 
in which the second order derivative (Hessian) of $f(\cdot)$ or second order statistics of gradients is incorporated in $\mathbf{G}$. 
In the Newton method,  $\mathbf{G}$ is chosen as the Hessian matrix of $f(\cdot)$ according to
\begin{align}
\mathbf{G}_{ij} = \frac{\partial^2 f(\bm{\theta})}{\partial \theta_i  \partial \theta_j}\;.
\end{align}
The Newton method approximates the loss function locally by a quadratic 
function, in which the Hessian matrix measures the curvature of the loss function.  This in turn yields better descent directions than the ones obtained by solely considering the gradient directions. In fact, under mild conditions, the Newton method converges to  a local minimum at a quadratic rate while gradient descent has a linear convergence 
rate~\cite{nocedal2006numerical}. Besides, the Newton method is invariant to affine re-parameterization (see the derivation in Appendix).

The concept of natural gradient~\cite{amari1998natural} provides another perspective by conditioning gradient step on the KL-divergence variations induced by the model output distribution. That is, any given step in the natural gradient method produces an equal amount of variation in terms of the KL-divergence. 
It is shown that the method of natural gradient is invariant to the parameterization of the model~\cite{pascanu2013revisiting}. 
Approximating the KL-divergence variations by its second order Taylor series, one arrives at a form that looks like the Newton method, albeit with
$\mathbf{G}$ becoming the Fisher information matrix
\begin{align}
\label{eqn:fisher_information}
\mathbf{G} =  \mathbb{E}_{\mathbf{x}\sim p(\mathbf{x}|\bm{\theta})} \left[ \frac{\partial \log p(\mathbf{x}|\bm{\theta})}{\partial \bm{\theta}}\frac{\partial \log p(\mathbf{x}|\bm{\theta})}{\partial \bm{\theta}}^\intercal \right]
\end{align}
where $p(\mathbf{x}|\bm{\theta})$ is the probabilistic model we try to optimize. 
It has been shown that the Fisher information matrix can be viewed as an approximated Hessian matrix~\cite{martens2014new}.  
Various studies suggest that natural gradient can have a better convergence rate than that of the gradient descent method (\eg, blind signal separation~\cite{amari1996new}, reinforcement learning~\cite{peters2008natural} and variational inference \cite{honkela2010approximate}). 

\newpage

For stochastic optimization, the AdaGrad algorithm~\cite{duchi2011adaptive}  incorporates the previously computed gradients to guide its current descent direction. In AdaGrad, $\mathbf{G}$ is the matrix square root of an outer product matrix of gradient vectors
\begin{align}
\mathbf{G} = \left(\sum_{\tau=1}^t \mathbf{g}_{\tau} \mathbf{g}_{\tau}^\intercal\right)^{\frac{1}{2}}
\end{align}
where $\mathbf{g}_{\tau} = \nabla_{\bm{\theta}}f_{\tau}(\bm{\theta}_{\tau})$ is the gradient estimated from a mini-batch of data at step $\tau$ and $t$ is the current step. The relationship between the matrix $\sum_{\tau=1}^t \mathbf{g}_{\tau} \mathbf{g}_{\tau}^\intercal$ 
and the Hessian matrix is discussed in \cite{hazan2007logarithmic}.

Despite their intriguing properties and fast convergence rates, the aforementioned second order optimization methods become overwhelming for high-dimensional $\bm{\theta}$. 
This is because construction and inversion of $\mathbf{G}$ have a time complexity of $O(d^3)$. 
As such, various studies resort to approximation techniques when it comes to high-dimensional problems.

A simple, yet effective approximation is the diagonal approximation, where one only keeps the diagonal elements of $\mathbf{G}$.  A classic method is the Levenberg-Marquardt 
algorithm~\cite{levenberg1944method,marquardt1963algorithm} which uses the diagonal approximation of the Hessian matrix. Its stochastic version for training neural networks is proposed in~\cite{becker1988improving}. The diagonal approximation of AdaGrad and its variants such as AdaDelta~\cite{zeiler2012adadelta}, RMSprop~\cite{hinton2012rmsprop} and Adam~\cite{kingma2014adam} 
have seen increasing popularity in training neural networks recently. 

The diagonal approximation of $\mathbf{G}$ amounts to setting an individual learning rate for each parameter. This balances the scale of parameters and often accelerates the convergence. 
Nevertheless, one wonders whether disregarding the correlation between the parameters, as captured by the off-diagonal elements of 
$\mathbf{G}$, has negative effects on the convergence speed. We will show that the answer to this question is a firm yes. A simple example is given in \textsection \ref{sec:parameter_dependency}. Experiments in \textsection \ref{sec:exp} also show the benefits of capturing off-diagonal elements of 
$\mathbf{G}$.

As such, in this paper we propose a new matrix approximation technique. The main idea is to split $\mathbf{G}$ into smaller blocks and approximate each block by one or two numbers. This, in return, will enable us 
to approximate the inverse of $\mathbf{G}$ by inverting small matrices, drastically reducing the time complexity while benefiting from the off-diagonal elements of $\mathbf{G}$. We incorporate our method into  AdaGrad  for training deep neural networks and empirically observe that the resulting algorithm outperforms AdaGrad with diagonal approximation in terms of the convergence speed.

\newpage

\section{Parameter Dependency}
\label{sec:parameter_dependency}

\begin{figure}[t!]
\centering
\subfloat[Original]{
\includegraphics[width=0.45\linewidth]{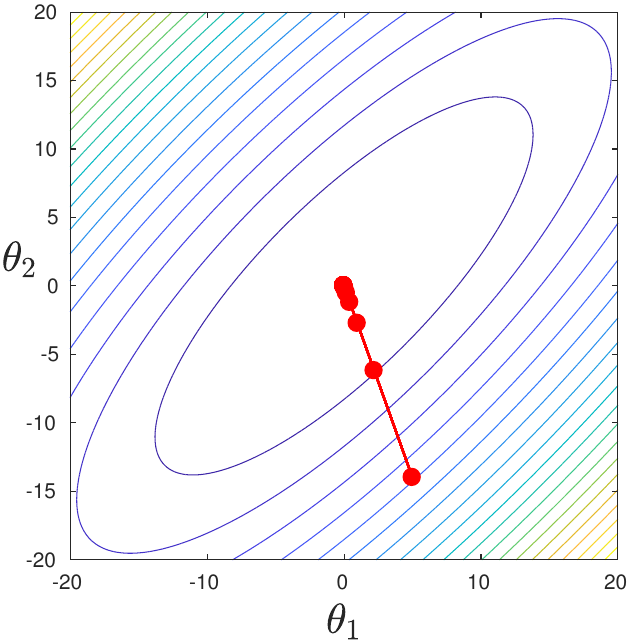}}
\subfloat[Transformed]{
\includegraphics[width=0.45\linewidth]{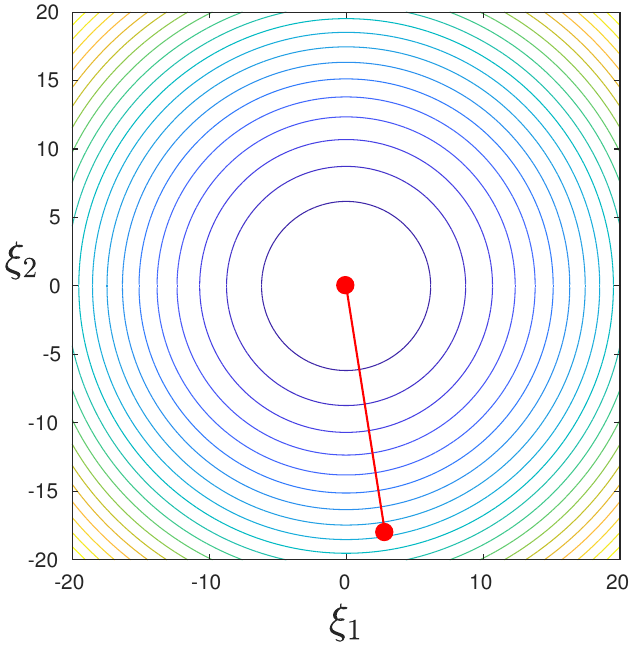}}
\caption{The effect of parametrization in convergence behavior of gradient descent (with line search). (a) Gradient descent moves in a zig-zag fashion 
as it approaches to the minimum. (b) By reparameterization the same problem, gradient descent finds the minimum in only one step (see the text for more details).}
\label{fig:gd}
\end{figure}

To show how the dependency of parameters affects the speed of gradient descent, consider a simple quadratic minimization problem
\begin{align}
\min_{\bm{\theta}}f(\bm{\theta}) = \frac{1}{2}\bm{\theta}^\intercal\mathbf{Q}\bm{\theta}\;
\end{align}
where 
\begin{align}
\bm{\theta} = [\theta_1, \theta_2]^\intercal, \quad \mathbf{Q} = 
\begin{bmatrix}
1.0 & -0.8 \\
-0.8 & 1.0
\end{bmatrix}.
\end{align} The eigenvalues of $\mathbf{Q}$ are $\lambda_1 = 1.8$ and $\lambda_2 = 0.2$.
The minimum is obtained at $\bm{\theta} = \mathbf{0}$. 
We note that the scale of $\theta_1$ equals that of $\theta_2$. However, the gradient of $\theta_1$ depends on the value of $\theta_2$ and vice versa as
\begin{align}
\nabla_{\bm{\theta}} f(\bm{\theta}) = \mathbf{Q}\bm{\theta}\;.
\end{align}

The gradient descent method, even with the optimal step size for each iteration, converges with a rate of 
$(\frac{\lambda_1 - \lambda_2}{\lambda_1 + \lambda_2})^2$ with such a parameterization.
However, with eigendecomposition $\mathbf{U}\mathbf{\Lambda}\mathbf{U}^\intercal = \mathbf{Q}$ and reparameterization with $\bm{\xi} = [\xi_1, \xi_2]^\intercal = \mathbf{\Lambda}^{\frac{1}{2}}\mathbf{U}\bm{\theta}$, then the optimization problem becomes 
\begin{align}
\min_{\bm{\xi}}f(\bm{\xi}) = \frac{1}{2} \bm{\xi}^\intercal\bm{\xi}\;.
\end{align}
The gradient of $\xi_1$ is not dependent on $\xi_2$ anymore and vice versa. As a result, gradient descent with the optimal step size converges to the minimum solution in only one iteration (see Figure~\ref{fig:gd} for an illustration).

The take home message from this textbook example is that \emph{algorithms that exploit the dependency between the parameters can prevail 
over the  ones that totally ignore such information}.

\newpage

\section{Block Mean Approximation}\label{sec:BMA}

We need the following definitions before introducing our proposed block mean approximation (BMA).

\begin{notation*}
$\mathbf{M}^{ij}$ denotes the ($i,j$)-th block of the matrix $\mathbf{M}$. $\mathbf{M}^{ij}_{mn}$ denotes the ($m,n$)-th element of the matrix $\mathbf{M}^{ij}$. 
\end{notation*}

\begin{defn}
For a diagonal matrix 
\begin{align}
\mathbf{\Lambda} = 
\begin{bmatrix}
\lambda_{1} &  0 & ... & 0 \\
0 & \lambda_{2}  & ... & 0 \\
... & ... & ... & ... \\
0 & 0 &  ... & \lambda_{L} 
\end{bmatrix}\;,
\end{align}
its diagonal expansion matrix with partition vector $\mathbf{s} = (s_1,...,s_L)$ is
\begin{align}
\bar{\mathbf{\Lambda}} = 
\begin{bmatrix}
\bar{\mathbf{\Lambda}}^{11} & 0 & ... & 0 \\
0 & \bar{\mathbf{\Lambda}}^{22} & ... & 0 \\
... & ... & ... & ... \\
0 &  0 &  ... & \bar{\mathbf{\Lambda}}^{LL} 
\end{bmatrix}
\end{align}
where each $\bar{\mathbf{\Lambda}}^{ii}$ is a diagonal matrix of size $s_i \times s_i$ with fixed diagonal elements of $\lambda_i$.
\end{defn}

\begin{defn}
For a matrix 
\begin{align}
\mathbf{B} =
\begin{bmatrix}
b_{11} & b_{12} & ... & b_{1L} \\
b_{21} & b_{22} & ... & b_{2L} \\
... & ... & ... & ... \\
b_{L1} & b_{L2} & ... & b_{LL} 
\end{bmatrix}\;,
\end{align}
its full expansion matrix with partition vector $\mathbf{s} = (s_1,...,s_L)$ is 
\begin{align}
\bar{\mathbf{B}} =
\begin{bmatrix}
\bar{\mathbf{B}}^{11} &  \bar{\mathbf{B}}^{12} & ... & \bar{\mathbf{B}}^{1L} \\
\bar{\mathbf{B}}^{21} &  \bar{\mathbf{B}}^{22} & ... & \bar{\mathbf{B}}^{2L} \\
... & ... & ... & ... \\
\bar{\mathbf{B}}^{L1} &  \bar{\mathbf{B}}^{L2} & ... & \bar{\mathbf{B}}^{LL}
\end{bmatrix}
\end{align}
where each block $\bar{\mathbf{B}}^{ij}$ is a matrix of constant value $b_{ij}$ and has $s_i$ number of rows.
\end{defn}

The above definitions are illustrated in Figure \ref{matrices}.

\begin{figure}[h!]
\centering
\subfloat[$\mathbf{\Lambda}$]{
\begin{tikzpicture}
\draw[step=0.2,gray!50,thin] (0,0) grid (0.6,0.6);
\draw[-,thick] (0,0) -- (0,0.6);
\draw[-,thick] (0,0) -- (0.6,0);
\draw[-,thick] (0.6,0) -- (0.6,0.6);
\draw[-,thick] (0,0.6) -- (0.6,0.6);

\filldraw[fill=blue!40, draw=gray] 
(0.0,0.6) rectangle (0.2,0.4);	
\filldraw[fill=black!30, draw=gray] 
(0.2,0.4) rectangle (0.4,0.2);	
\filldraw[fill=red, draw=gray] 
(0.4,0.2) rectangle (0.6,0.0);	

\end{tikzpicture}}
\hspace{0.1in}
\subfloat[$\bar{\mathbf{\Lambda}}$]{
\begin{tikzpicture}
\draw[step=0.2,gray!50,thin] (0,0) grid (2,2);
\draw[-,thick] (0,0) -- (0,2);
\draw[-,thick] (0,0) -- (2,0);
\draw[-,thick] (2,0) -- (2,2);
\draw[-,thick] (0,2) -- (2,2);

\filldraw[fill=blue!40, draw=gray] 
(0.0,2.0) rectangle (0.2,1.8);	
\filldraw[fill=blue!40, draw=gray] 
(0.2,1.8) rectangle (0.4,1.6);	
\filldraw[fill=black!30, draw=gray] 
(0.4,1.6) rectangle (0.6,1.4);	
\filldraw[fill=black!30, draw=gray] 
(0.6,1.4) rectangle (0.8,1.2);	
\filldraw[fill=black!30, draw=gray] 
(0.8,1.2) rectangle (1.0,1.0);	
\filldraw[fill=black!30, draw=gray] 
(1.0,1.0) rectangle (1.2,0.8);	
\filldraw[fill=black!30, draw=gray] 
(1.2,0.8) rectangle (1.4,0.6);	
\filldraw[fill=red, draw=gray] 
(1.4,0.6) rectangle (1.6,0.4);	
\filldraw[fill=red, draw=gray] 
(1.6,0.4) rectangle (1.8,0.2);	
\filldraw[fill=red, draw=gray] 
(1.8,0.2) rectangle (2.0,0.0);	

\end{tikzpicture}}
\hspace{0.1in}
\subfloat[$\mathbf{B}$]{
\begin{tikzpicture}
\draw[step=0.2,gray!50,thin] (0,0) grid (0.6,0.6);
\draw[-,thick] (0,0) -- (0,0.6);
\draw[-,thick] (0,0) -- (0.6,0);
\draw[-,thick] (0.6,0) -- (0.6,0.6);
\draw[-,thick] (0,0.6) -- (0.6,0.6);

\filldraw[fill=yellow, draw=gray] 
(0.0,0.0) rectangle (0.2,0.2);
\filldraw[fill=blue!80, draw=gray] 
(0.2,0.0) rectangle (0.4,0.2);
\filldraw[fill=gray!20, draw=gray] 
(0.4,0.0) rectangle (0.6,0.2);

\filldraw[fill=black!40, draw=gray] 
(0.0,0.2) rectangle (0.2,0.4);
\filldraw[fill=green!50, draw=gray] 
(0.2,0.2) rectangle (0.4,0.4);
\filldraw[fill=red, draw=gray] 
(0.4,0.2) rectangle (0.6,0.4);

\filldraw[fill=red!90, draw=gray] 
(0.0,0.4) rectangle (0.2,0.6);
\filldraw[fill=black!60, draw=gray] 
(0.2,0.4) rectangle (0.4,0.6);
\filldraw[fill=blue!50, draw=gray] 
(0.4,0.4) rectangle (0.6,0.6);
\end{tikzpicture}}
\hspace{0.1in}
\subfloat[$\bar{\mathbf{B}}$]{
\begin{tikzpicture}

\filldraw[fill=yellow, draw=gray] 
(0.0,0.0) rectangle (0.4,0.6);
\filldraw[fill=blue!80, draw=gray] 
(0.4,0.0) rectangle (1.4,0.6);
\filldraw[fill=gray!20, draw=gray] 
(1.4,0.0) rectangle (2.0,0.6);

\filldraw[fill=black!40, draw=gray] 
(0.0,0.6) rectangle (0.4,1.6);
\filldraw[fill=green!50, draw=gray] 
(0.4,0.6) rectangle (1.4,1.6);
\filldraw[fill=red, draw=gray] 
(1.4,0.6) rectangle (2.0,1.6);

\filldraw[fill=red!90, draw=gray] 
(0.0,1.6) rectangle (0.4,2.0);
\filldraw[fill=black!60, draw=gray] 
(0.4,1.6) rectangle (1.4,2.0);
\filldraw[fill=blue!50, draw=gray] 
(1.4,1.6) rectangle (2.0,2.0);

\draw[step=0.2,gray!50,thin] (0,0) grid (2,2);
\draw[-,thick] (0,0) -- (0,2);
\draw[-,thick] (0,0) -- (2,0);
\draw[-,thick] (2,0) -- (2,2);
\draw[-,thick] (0,2) -- (2,2);

\end{tikzpicture}}
\caption{Expansion matrices. (a) Diagonal  matrix $\mathbf{\Lambda}$. (b) Diagonal expansion of 
$\mathbf{\Lambda}$. (c) Full matrix $\mathbf{B}$. (d) Full expansion of $\mathbf{B}$. The partition vector in both cases 
is $\mathbf{s}=(2,5,3)$.} 
\label{matrices}
\end{figure}
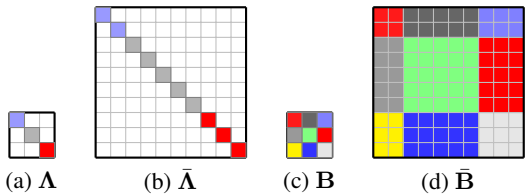

\begin{defn}
The block mean approximation of a matrix $\mathbf{M}$ with the partition vector $\mathbf{s}$ is
\begin{align}
\widehat{\mathbf{M}} = \bar{\mathbf{\Lambda}} + \bar{\mathbf{B}} \approx \mathbf{M} 
\end{align}
where $\bar{\mathbf{\Lambda}}$ and $\bar{\mathbf{B}}$ are the diagonal and full expansion matrices with partition vector $\mathbf{s}$, respectively.
\end{defn}

\begin{figure}[t!]
\centering	
\subfloat[Original]{
\frame{
\includegraphics[width=0.35\linewidth]{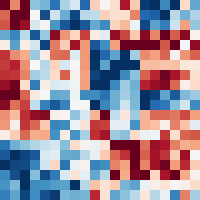}}}
\hspace{0.2in}
\subfloat[Approximate]{
\frame{
\includegraphics[width=0.35\linewidth]{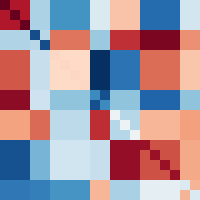}}}
\caption{Block mean approximation of a matrix.}
\label{fig:bma}
\end{figure}

The block mean approximation, as illustrated in Figure \ref{fig:bma}, allows one to efficiently store a big matrix as only $\mathbf{\Lambda}$, $\mathbf{B}$ and a partition vector $\mathbf{s}$ need to be kept. 
Given a square matrix, its optimal block mean approximation under Frobenius norm is given by the following.
\begin{prop}
The optimal block mean approximation of $\mathbf{M}$ with the partition vector $\mathbf{s}$ according to the Frobenius norm 
\begin{align}
\min_{\bar{\mathbf{\Lambda}},\bar{\mathbf{B}}} \| \bar{\mathbf{\Lambda}}+\bar{\mathbf{B}} - \mathbf{M}  \|_F^2\;
\end{align}
is given by
\begin{align}
b_{ij} &= 
\begin{cases}
0,  &i = j, s_i = 1, \\
\frac{\sum_{mn} \mathbf{M}^{ii}_{mn} - \sum_{m} \mathbf{M}^{ii}_{mm}}{s_{i}(s_{i}-1)},  &i = j, s_i \neq 1, \\
\frac{\sum_{mn}\mathbf{M}^{ij}_{mn}}{s_{i}s_{j}}, &i \neq j, 
\end{cases} \label{b_ij} \\
\lambda_{i} &= \frac{1}{s_{i}}\sum_m \mathbf{M}^{ii}_{mm} - b_{ii}\;. \label{lambda_i}
\end{align}
\end{prop}

Proposition 1 can be understood as follows. According to (\ref{b_ij}) and (\ref{lambda_i}), the non-diagonal block $\mathbf{M}^{ij}$ is approximated by $\widehat{\mathbf{M}}^{ij} = \mathbf{B}^{ij}$, whose value is the mean value of $\mathbf{M}^{ij}$, which minimizes the Frobenius form. The diagonal block $\mathbf{M}^{ii}$  is approximated by $\widehat{\mathbf{M}}^{ii} = \mathbf{\Lambda}^{ii} + \mathbf{B}^{ii}$, whose diagonal values are equal to the mean diagonal values of $\mathbf{M}^{ii}$ and its off-diagonal values are equal to the mean off-diagonal values of $\mathbf{M}^{ii}$, which again minimizes the Frobenius form.

The power of block mean approximation lies in the ease of computing its inverse and inverse square root matrices, as shown by the following theorems. All the proofs are relegated to the Appendix.

\newpage

\begin{theorem}
For invertible matrix $\bar{\mathbf{\Lambda}}+\bar{\mathbf{B}}$, where $\bar{\mathbf{\Lambda}}$ and $\bar{\mathbf{B}}$ are the diagonal and full expansion of $\mathbf{\Lambda}$ and $\mathbf{B}$ with respect to the partition vector $\mathbf{s}$, 
\begin{align}
(\bar{\mathbf{\Lambda}}+\bar{\mathbf{B}})^{-1} = \bar{\mathbf{\Lambda}}^{-1} + \bar{\mathbf{D}}
\end{align}
 where $\bar{\mathbf{D}}$ is the full expansion matrix with partition vector $\mathbf{s}$ of 
\begin{align}
\mathbf{D} =  (\mathbf{\Lambda} \mathbf{S} + \mathbf{S}\mathbf{B}\mathbf{S} )^{-1} - (\mathbf{\Lambda}\mathbf{S})^{-1} 
\end{align}
where $\mathbf{S}= \text{\normalfont diag}(\mathbf{s})$.
\end{theorem}

\begin{theorem}
For invertible matrix $\bar{\mathbf{\Lambda}}+\bar{\mathbf{B}}$, where $\bar{\mathbf{\Lambda}}$ and $\bar{\mathbf{B}}$ are the diagonal and full expansion of $\mathbf{\Lambda}$ and $\mathbf{B}$ with respect to the partition vector $\mathbf{s}$,
\begin{align}
(\bar{\mathbf{\Lambda}}+\bar{\mathbf{B}})^{-\frac{1}{2}} = \bar{\mathbf{\Lambda}}^{-\frac{1}{2}} + \bar{\mathbf{D}}
\end{align}
 where $\bar{\mathbf{D}}$ is the full expansion matrix with partition vector $\mathbf{s}$ of 
\begin{align}
\mathbf{D} = \mathbf{S}^{-\frac{1}{2}}\left[(\mathbf{\Lambda} + \mathbf{S}^{\frac{1}{2}}\mathbf{B}\mathbf{S}^{\frac{1}{2}} )^{-\frac{1}{2}} - \mathbf{\Lambda}^{-\frac{1}{2}}\right]\mathbf{S}^{-\frac{1}{2}} 
\end{align}
where $\mathbf{S}= \text{\normalfont diag}(\mathbf{s})$.
\end{theorem}

The importance of the above theorems can be understood by noting that inverting  $\mathbf{G} \in \mathbb{R}^{d \times d}$ by splitting it into $L\times L$ blocks
has a complexity of $O(L^3)$, which can be significantly faster than $O(d^3)$ flops required to obtain $\mathbf{G}^{-1}$ or $\mathbf{G}^{-\frac{1}{2}}$.

\section{BMA for Neural Networks}
As block mean approximation divides a matrix into $L\times L$ blocks, a natural question to ask is \emph{how much prior knowledge is needed to
determine such a block structure?} 
In training neural networks, we recommend to group the parameters in each layer (or even each unit) into a block such that $\mathbf{G}$ with block mean approximation represents the scale and dependency between layers (or units). This comes naturally as the output and gradient of each layer often depends on one another. 

There are several work that use heuristics to set an individual learning rate for layers of a deep network~\cite{singh2015layer,you2017scaling,abu2017proportionate}. Nevertheless, even with such heuristics, the dependency between parameters is ignored. 
In contrast, the BMA gives a principled way to set learning rates per layers in a deep network while capturing the dependency between layers.

In Figure~\ref{eFIM}, we compute the empirical Fisher information matrix $\mathbb{E}_{\mathbf{x}\sim p_{\text{data}}}[ \frac{\partial \log p(\mathbf{x}|\bm{\theta})}{\partial \bm{\theta}}\frac{\partial \log p(\mathbf{x}|\bm{\theta})}{\partial \bm{\theta}}^\intercal ]$ for a convolutional neural network (described in Table 1) and compare BMA with different block partitions. Even for the finest approximation in Figure~\ref{eFIM}, only $45\times 45$ matrices need to be constructed and inverted.

\newpage

\begin{figure}
\centering
\subfloat[Original]{\includegraphics[width=0.5\linewidth]{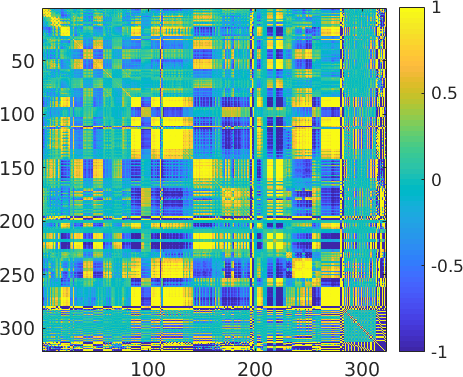}}
\subfloat[$10\times 10$ blocks]{\includegraphics[width=0.5\linewidth]{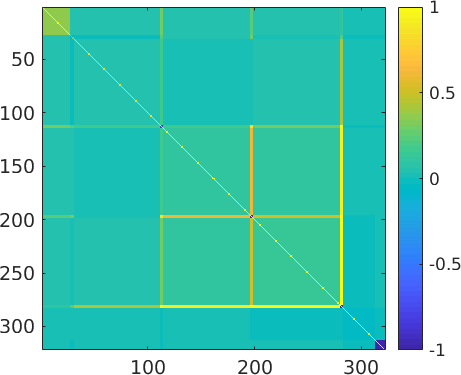}}

\subfloat[$27\times 27$ blocks]{\includegraphics[width=0.5\linewidth]{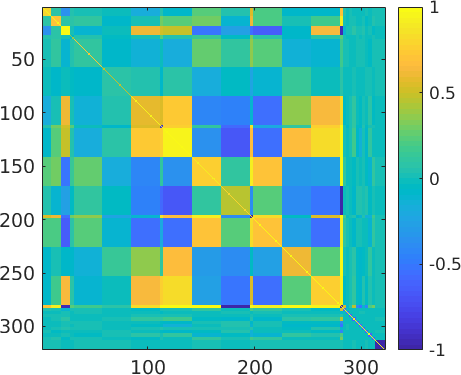}}
\subfloat[$45\times 45$ blocks]{\includegraphics[width=0.5\linewidth]{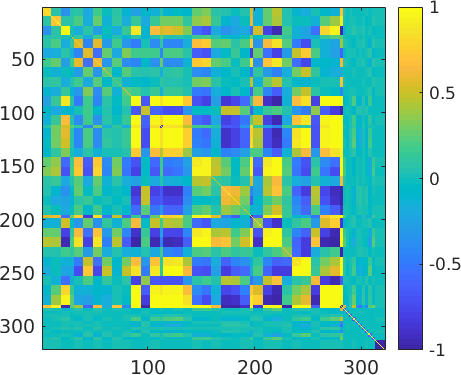}}

\caption{Block mean approximations of the empirical Fisher information matrix of a neural network model. The values are normalized for better visualization. (a) The original matrix has a size of $322\times 322$. (b) Each block represents weights or bias in a layer. MSE (mean square error) = 0.1036. (c) Each block represents weights or bias in a unit. MSE=0.1024. (d) Each block represents a group of weights or bias in a unit. MSE=0.0998.}
\label{eFIM}
\end{figure}

\section{Related Work}
There are several generic matrix approximation techniques which have been applied to second order optimization methods. Aside from the  diagonal approximation, 
block diagonal approximation has been applied to AdaGrad~\cite{duchi2011adaptive} and the Gauss-Newton method~\cite{botev2017practical}. 
Low rank approximation has been applied to natural gradient~\cite{roux2008topmoumoute} and AdaGrad~\cite{krummenacher2016scalable}. 
Kronecker approximation \cite{martens2015optimizing,grosse2016kronecker}, sparse approximation \cite{grosse2015scaling} and quasi-diagonal approximation \cite{ollivier2015riemannian} have been applied to natural gradient.

We stress that the block mean approximation takes into account all the elements of $\mathbf{G}$, 
while the diagonal and block diagonal approximations neglect most of the elements of $\mathbf{G}$.
Furthermore, the block mean approximation does not force any low-rank assumption, which cannot be guaranteed in general. For example, when $\mathbf{G} = \sigma \mathbf{I}$ for $\sigma > 0$, $\mathbf{G}$ has the singular values all equal to $\sigma$ and thus does not have a low-rank structure. The other advantage of the BMA is its ease of implementation. The approximate matrix $\widehat{\mathbf{G}}$ does not need to be constructed explicitly in general. 
This is shown in \textsection~\ref{sec:adagrad_bma}.

\section{AdaGrad with BMA}\label{sec:adagrad_bma}

For unconstrained stochastic optimization problems, the full version of AdaGrad \cite{duchi2011adaptive} has the following update rule,
\begin{align}
\bm{\theta}_{t+1} = \bm{\theta}_{t} - \eta \mathbf{G}_t^{-1} \mathbf{g}_{t}
\label{ada_grad}
\end{align}
where $\mathbf{G}_t = \mathbf{H}_t^{\frac{1}{2}}$ and $\mathbf{H}_t = \sum_{\tau=1}^t \mathbf{g}_{\tau} \mathbf{g}_{\tau}^\intercal$.

We approximate the gradient outer product matrix in the following form
\begin{align}
\widehat{\mathbf{H}}_t = \mathbf{Z}_t\mathbf{F}_t\mathbf{Z}_t \approx \mathbf{H}_t 
\label{H_t}
\end{align}
where $\mathbf{Z}_t = \text{diag}(\mathbf{z}_t)$ is a diagonal matrix and $\mathbf{F}_t$ is a positive definite matrix. As AdaGrad requires the computation of $\mathbf{H}_t^{-\frac{1}{2}}$, one needs to choose $\mathbf{F}_t$ such that $\mathbf{F}_t^{-\frac{1}{2}}$ is easy to obtain.

To derive the algorithm, we first define the following notations.
\begin{notation*}
We divide vector $\mathbf{g}$ into $L$ blocks such that $\mathbf{g} = (\mathbf{g}^1,...,\mathbf{g}^L)$. $\mathbf{g}^i$ denotes the $i$-th block of vector $\mathbf{g}$. $g^i_m$ denotes the $m$-th element of vector $\mathbf{g}^i$. $({g}^i_m)^2$ denotes the square value of ${g}^i_m$. $\mathbf{g}_t$ denote a vector at step $t$. $\mathbf{g}_t^i$ denote the $i$-th block of $\mathbf{g}_t$. $g_{t,m}^i$ denote the $m$-th element of $\mathbf{g}_t^i$. $\mathbf{g}^2$ denote elementwise square of $\mathbf{g}$. $\sqrt{\mathbf{g}}$ denotes elementwise square root of $\mathbf{g}$. $\mathbf{g} + a$ denotes each element of $\mathbf{g}$ is added by scalar $a$.
\end{notation*}

\subsection{Diagonal Approximation}
If we set $\mathbf{z}_t$ with elements
\begin{align}
z_{t,i} = \sqrt{\sum_{\tau=1}^t (g_{\tau,i})^2}
\label{Z}
\end{align}
and  $\mathbf{F}_t = \mathbf{I}$, $\widehat{\mathbf{H}}_t$ is reduced to diagonal approximation. 

\subsection{Block Mean Approximation}
In order to capture the off-diagonal elements of $\mathbf{H}$, we approximate $\mathbf{F}_t$ with the block mean approximation proposed in \textsection \ref{sec:BMA}.
We set $\mathbf{z}_t$ as in (\ref{Z})
and seek 
\begin{align}
\mathbf{F}_t = \bar{\mathbf{\Lambda}} + \bar{\mathbf{B}} \approx \mathbf{Z}_t^{-1}(\sum_{\tau=1}^t\mathbf{g}_{\tau}\mathbf{g}_{\tau}^\intercal)\mathbf{Z}_t^{-1}.
\label{F_t}
\end{align}

With the block mean approximation, (\ref{H_t}) can be interpreted as follows: $\mathbf{Z}_t$ gives an individual learning rate of each parameter and $\mathbf{F}_t$ captures the dependency between the groups of parameters.

To realize block mean approximation, we partition the parameters into $L$ groups.

\newpage

Let $\mathbf{g}_t = (\mathbf{g}_t^1, \mathbf{g}_t^2,...,\mathbf{g}_t^L)$. 
Define $\mathbf{u}_t$ and $\mathbf{v}_t$ with
\begin{align}
u_{t,i} =\sum_{m}  g_{t,m}^i, \quad 
v_{t,i} = \sum_{m} z_{t,m}^i,
\end{align}
respectively. Let $\mathbf{U}_t = \sum_{\tau=1}^t \mathbf{u}_{\tau}\mathbf{u}_{\tau}^\intercal$.
To approximate $\mathbf{Z}_t^{-1}(\sum_{\tau=1}^t\mathbf{g}_{\tau}\mathbf{g}_{\tau}^\intercal)\mathbf{Z}_t^{-1}$, we choose $\bar{\mathbf{\Lambda}}$ and $\bar{\mathbf{B}}$ to be the expansion matrices of $\mathbf{\Lambda}$ and $\mathbf{B}$
\begin{align}
\mathbf{\Lambda} &= \mathbf{I}, \quad \mathbf{B} = \mathbf{S}^{-\frac{1}{2}}\frac{\mathbf{U}_t - \text{diag}(\mathbf{U}_t)}{\mathbf{v}_t\mathbf{v}_t^\intercal} \mathbf{S}^{-\frac{1}{2}} 
\end{align}
where the division is elementwise. Based on Theorem 2, the inverse square root is $\mathbf{I} + \bar{\mathbf{D}}$ where $\bar{\mathbf{D}}$ is the expansion matrix of 
\begin{align}
\mathbf{D} = \mathbf{S}^{-\frac{1}{2}}\left[\left(\mathbf{I} + \frac{\mathbf{U}_t - \text{diag}(\mathbf{U}_t)}{\mathbf{v}_t\mathbf{v}_t^\intercal}\right)^{-\frac{1}{2}} - \mathbf{I}\right]\mathbf{S}^{-\frac{1}{2}}.  \label{D}
\end{align}
The inverse square root in (\ref{D}) can be computed as follows.
Let $\mathbf{RVR}^\intercal$ be the eigendecomposition of a matrix, then its inverse square root is $\mathbf{R}\mathbf{V}^{-\frac{1}{2}}\mathbf{R}^\intercal$. In case where the eigenvalues are zeros or too small, we clamp the eigenvalues to have a minimal value before computing $\mathbf{V}^{-\frac{1}{2}}$.

We call the above algorithm AdaGrad-BMA, which is summarized in Algorithm 1.  
The eigendecomposition has time complexity $O(L^3)$. Therefore, for parameters of dimension $d$ and partitioned into $L$ blocks, AdaGrad-BMA has time complexity $O(L^3 + d)$ per iteration.

\begin{algorithm}[t]
   \caption{AdaGrad-BMA}
   \label{alg:example}
\begin{algorithmic}[1]
   \STATE {\bfseries Input:} Objective function $f(\bm{\theta})$ with parameters $\bm{\theta}$
   \STATE {\bfseries Input:} A partition of parameters $\bm{\theta}=\{\bm{\theta}^1,...,\bm{\theta}^L\}$
   \STATE {\bfseries Input:} Partition vector $\mathbf{s}$ that $s_i=$  size($\bm{\theta}^i$)
   \STATE {\bfseries Input:} Hyperparameters $\eta$ and $\epsilon$
   \STATE Initialize $\mathbf{U} = \mathbf{0}$, $\mathbf{v} = \mathbf{0}$, $\mathbf{r} = \mathbf{0}$, $\mathbf{S}=\text{diag}(\mathbf{s})$
   \FOR{$t=1$ {\bfseries to} $T$}
   \STATE $\mathbf{g} \leftarrow \nabla_{\bm{\theta}}f_t(\bm{\theta}_t)$    
   \STATE $\mathbf{r} \leftarrow \mathbf{r} + \mathbf{g}^2$ 
   \STATE $\mathbf{z} \leftarrow \sqrt{\mathbf{r}+\epsilon}$   
   \FOR{$i=1$ {\bfseries to} $L$}   
   \STATE   $u_i \leftarrow  \sum_m {g}^i_{m} $
   \STATE   $v_i \leftarrow  \sum_m {z}^i_{m}  $
   \ENDFOR   
   \STATE $\mathbf{U} \leftarrow \mathbf{U} + \mathbf{u}\mathbf{u}^\intercal$           
   \STATE Compute $\mathbf{D}$ according to (\ref{D})
   \STATE $\mathbf{g} \leftarrow \mathbf{Z}^{-\frac{1}{2}} \mathbf{g}$        
   \FOR{$i=1$ {\bfseries to} $L$}
   \STATE $\mathbf{g}^i \leftarrow \mathbf{g}^i + \sum_j\mathbf{D}_{ij}{u}_j$  
   \ENDFOR        
   \STATE $\mathbf{g} \leftarrow \mathbf{Z}^{-\frac{1}{2}} \mathbf{g}$     
   \STATE $\bm{\theta} \leftarrow \bm{\theta} - \eta \mathbf{g} $
   \ENDFOR        
\end{algorithmic}
\end{algorithm}

\newpage

\section{Experiments}\label{sec:exp}

We evaluate AdaGrad-BMA in training convolutional neural networks, against the full version of AdaGrad (AdaGrad-full) and AdaGrad with diagonal approximation (AdaGrad-diag). For AdaGrad-BMA, we group the weights and the bias parameters separately for each layer. For a model of $l$  convolution or fully connected layers, we partition $\mathbf{G}$ into $L\times L$ blocks with BMA where $L = 2l$.

The experiments are done on two standard datasets MNIST and CIFAR-10. We use two simple models: small and large, as described in Table 1 and 2.
We choose the architecture of the small model to ensure that AdaGrad-full is applicable. For the large model, AdaGrad-full is too computationally expensive to use.
Each convolution layer has kernel size $3\times 3$, stride 1 and zero-padding 1. Each convolution layer is followed by a hyperbolic tangent activation function. 
The number of parameters of each model is listed in Table 3. As MNIST and CIFAR-10 have different input image size, the fully connected layer in each model has different size of inputs, therefore resulting in different number of parameters.

For each algorithm on each dataset, we tried learning rates $\eta \in \{1,10^{-1},10^{-2},10^{-3},10^{-4}\}$ and report the best performance achieved by the algorithm on the dataset.

The results are shown in Figure \ref{fig:results}, from which we can see AdaGrad-BMA outperforms AdaGrad-diag and achieves similar speed of convergence to AdaGrad-full.
The comparison of runtime of each iteration on MNIST is shown in Table \ref{runtime}. 
Although AdaGrad-BMA has longer runtime than AdaGrad-diag for each iteration, in practice one can update $\mathbf{G}^{-1}$ for every several steps to amortize the cost.

The code for the experiments is included in the supplementary materials.

\begin{table}[h!]
\centering
\begin{scriptsize}
\parbox{.45\linewidth}{
\centering
\caption{Small model}
\vspace{0.1in}
\begin{tabular}{|c|}
\hline
Conv 3x3, 3 \\
Max Pooling 2x2 \\
Conv 3x3, 3 \\
Max Pooling 2x2 \\
Conv 3x3, 3 \\
Max Pooling 2x2 \\
Conv 3x3, 3 \\
Max Pooling 2x2 \\
Fully Connected, 10 \\
Softmax, 10 \\
\hline
\end{tabular}}
\parbox{.45\linewidth}{
\centering
\caption{Large model}
\vspace{0.1in}
\begin{tabular}{|c|}
\hline
Conv 3x3, 32 \\
Conv 3x3, 32 \\
Conv 3x3, 32 \\
Conv 3x3, 32 \\
Max Pooling 2x2 \\
Conv 3x3, 32 \\
Conv 3x3, 32 \\
Conv 3x3, 32 \\
Conv 3x3, 32 \\
Max Pooling 2x2 \\
Conv 3x3, 32 \\
Conv 3x3, 32 \\
Conv 3x3, 32 \\
Conv 3x3, 32 \\
Max Pooling 2x2 \\
Conv 3x3, 32 \\
Conv 3x3, 32 \\
Conv 3x3, 32 \\
Conv 3x3, 32 \\
Max Pooling 2x2 \\
Fully Connected, 10 \\
Softmax, 10 \\
\hline
\end{tabular}}
\end{scriptsize}
\end{table}

\begin{table}[h!]
\centering
\caption{Model parameters}
\vspace{0.1in}
\begin{scriptsize}
\begin{tabular}{|c|c|c|c|}
\hline 
& MNIST &  CIFAR-10 \\
\hline 
Small model    & 322     & 466 \\
Large model    & 139370  & 140906 \\
\hline
\end{tabular}
\end{scriptsize}
\end{table}

\begin{table}[h!]
\centering
\caption{Runtime comparison (ms/iteration)}
\vspace{0.1in}
\begin{scriptsize}
\begin{tabular}{|c|c|c|c|}
\hline 
& Small model &  Large model \\
\hline 
AdaGrad-full    & 16.85     & - \\
AdaGrad-diag    & 5.70    & 7.55 \\
AdaGrad-BMA     & 10.07     & 19.12 \\
\hline
\end{tabular}
\end{scriptsize}
\label{runtime}
\end{table}

\begin{figure}[h!]
\centering
\subfloat[MNIST, small model]{\includegraphics[width=0.45\linewidth]{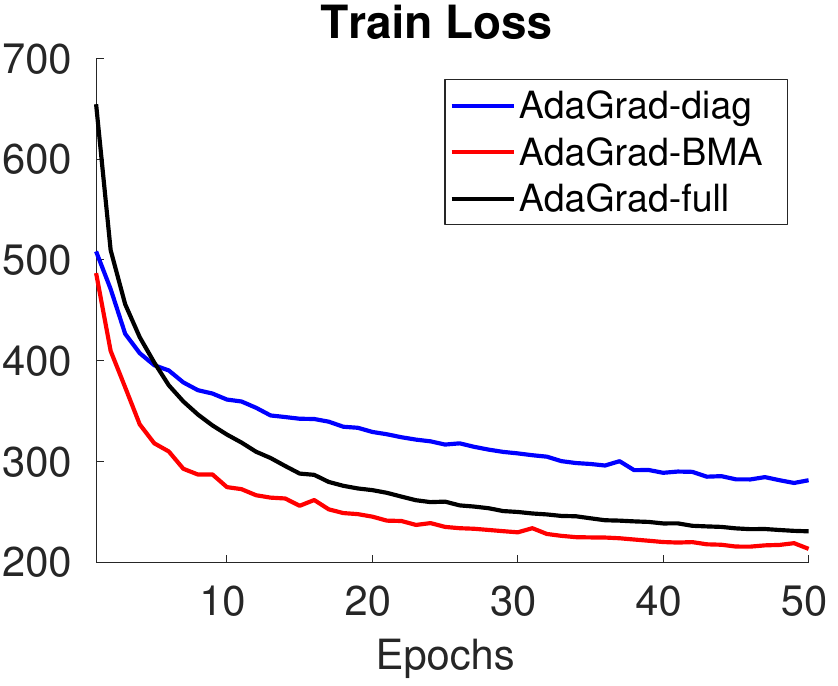}}
\subfloat[MNIST, small model]{\includegraphics[width=0.45\linewidth]{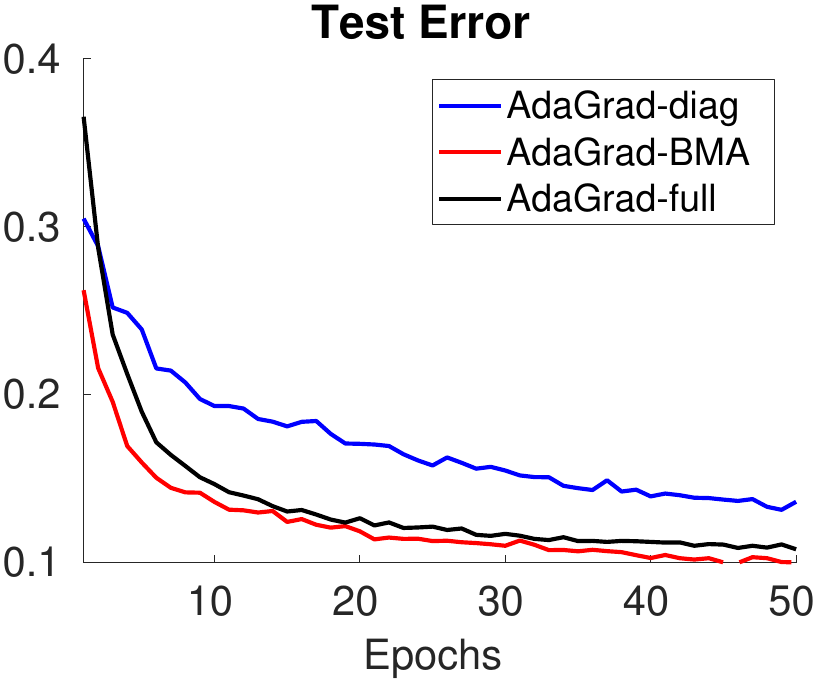}}

\subfloat[MNIST, large model]{\includegraphics[width=0.45\linewidth]{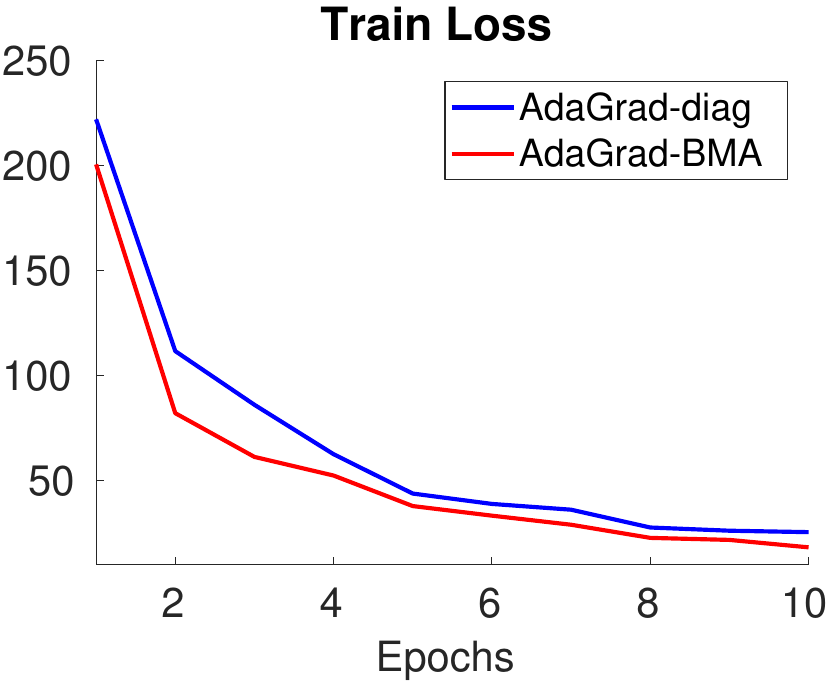}}
\subfloat[MNIST, large model]{\includegraphics[width=0.45\linewidth]{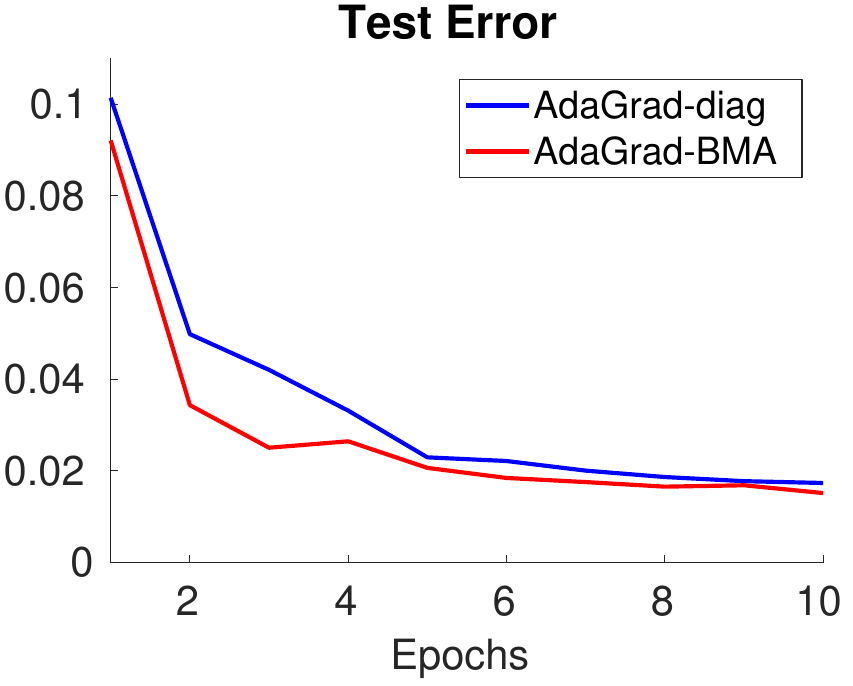}}

\subfloat[CIFAR-10, small model]{\includegraphics[width=0.45\linewidth]{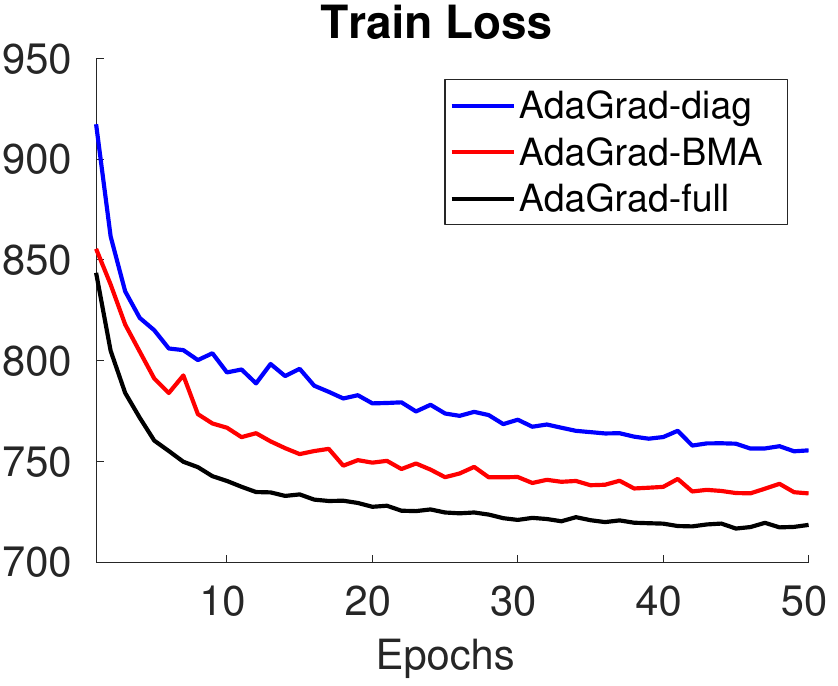}}
\subfloat[CIFAR-10, small model]{\includegraphics[width=0.45\linewidth]{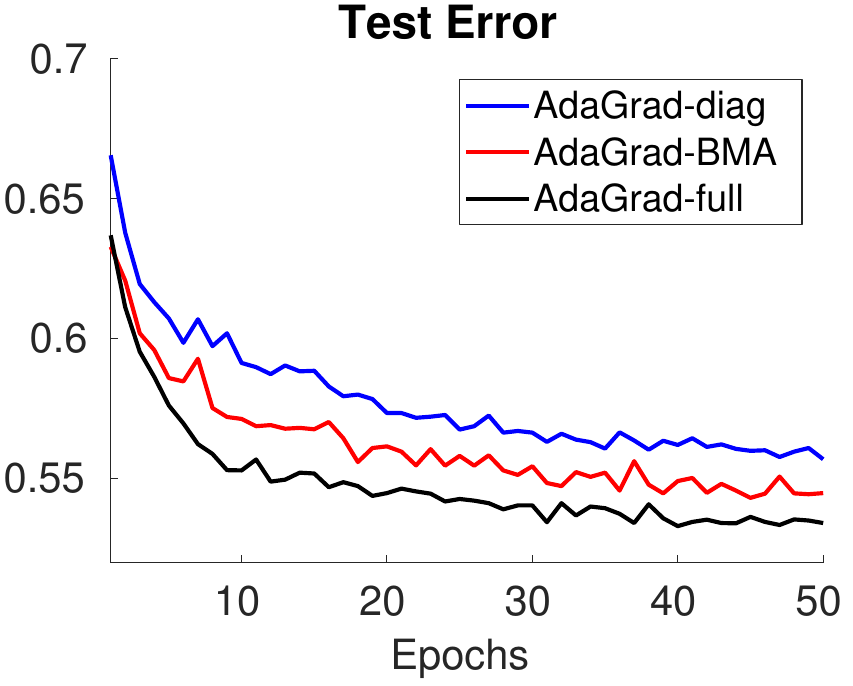}}

\subfloat[CIFAR-10, large model]{\includegraphics[width=0.45\linewidth]{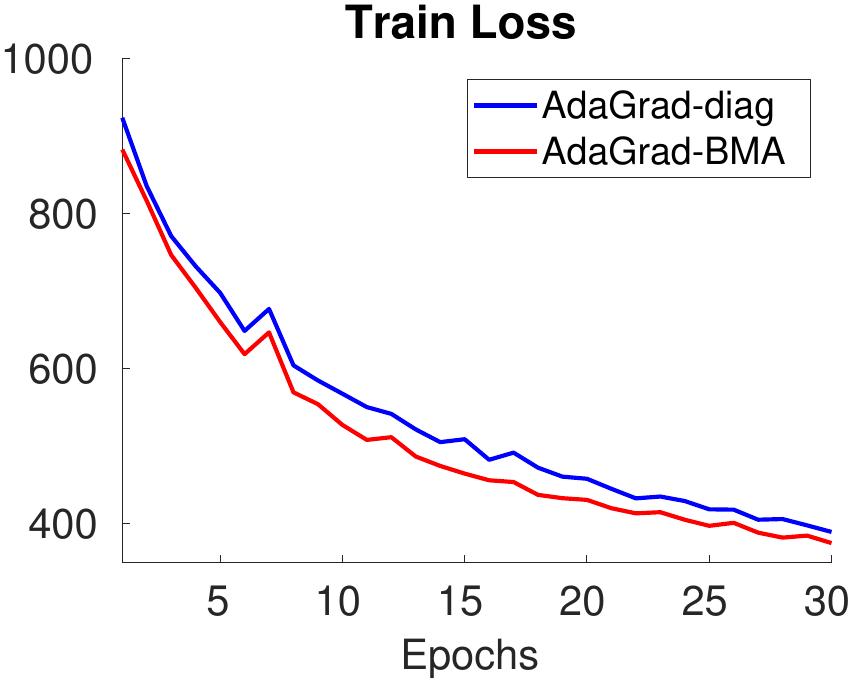}}
\subfloat[CIFAR-10, large model]{\includegraphics[width=0.45\linewidth]{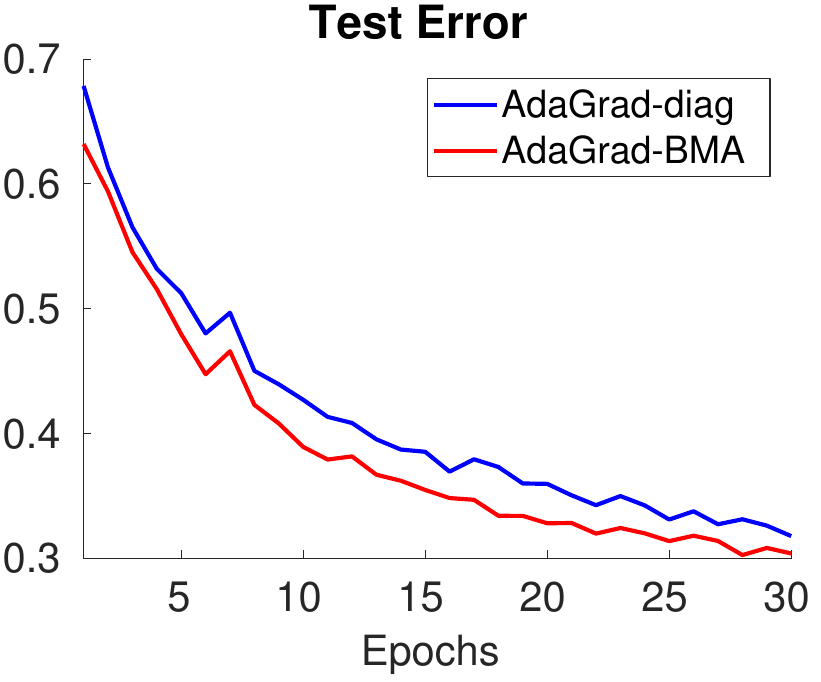}}
\caption{Performance of AdaGrad and its approximations on two standard datasets.}
\label{fig:results}
\end{figure}

\clearpage

\section{Discussions}

In this paper, we propose a new matrix approximation method which allows efficient storage and computation of matrix inverse and inverse square root.
The method is applied to AdaGrad and achieves promising results.

In the numerical linear algebra literature, there are two work relevant but different from ours.
\cite{chow1997approximate} proposed an approximate inverse technique which generates sparse solutions for block partitioned matrices. However, in our method the exact inverse of the block mean approximation matrix can be computed, as proved in Theorem 1 and 2. Our method does not assume sparse solutions either. \cite{guillaume2003block} proposed  an approximate inverse technique which incorporates block constant structure in the preconditioning matrix. This is different from ours as in our method the block constant structure is incorporated the approximated matrix and its inverse and the analytic solution of the inverse is explicitly given.

Second order optimization methods for training neural networks have many theoretical appeals, as discussed in this paper and reviewed in \cite{shepherd2012second,martens2016second}.  We hope our method makes one more step towards their practical implementation.

\newpage

\section*{Appendix}
\setcounter{prop}{0}
\setcounter{theorem}{0}

To show how Newton method is invariant of affine re-parameterization, let $\nabla^2$ be the Hessian operator, $\bm{\xi}=\mathbf{A}\bm{\theta}$ and $g(\bm{\theta}) = f(\mathbf{A}\bm{\theta})$ where $\mathbf{A}$ is an invertible square matrix. 
\begin{align}
\bm{\theta}_{t+1} &= \bm{\theta}_{t} - \eta (\nabla^2_{\bm{\theta}}g(\bm{\theta}))^{-1}\nabla_{\bm{\theta}}g(\bm{\theta}_{t}) \\
&=\bm{\theta}_{t} - \eta (\mathbf{A}^\intercal\nabla^2_{\bm{\xi}}f(\mathbf{A}\bm{\theta}_{t})\mathbf{A})^{-1}\mathbf{A}^\intercal\nabla_{\bm{\xi}}f(\mathbf{A}\bm{\theta}_{t}) \\
&=\bm{\theta}_{t} - \eta \mathbf{A}^{-1}(\nabla^2_{\bm{\xi}}f(\mathbf{A}\bm{\theta}_{t}))^{-1}\nabla_{\bm{\xi}}f(\mathbf{A}\bm{\theta}_{t}) 
\end{align}
we obtain
\begin{align}
\mathbf{A}\bm{\theta}_{t+1} &= \mathbf{A}\bm{\theta}_{t} - \eta (\nabla_{\bm{\xi}}^2f(\mathbf{A}\bm{\theta}_{t}))^{-1}\nabla_{\bm{\xi}}f(\mathbf{A}\bm{\theta}_{t}) 
\end{align}
which is equivalent to 
\begin{align}
\bm{\xi}_{t+1} &= \bm{\xi}_{t} - \eta (\nabla_{\bm{\xi}}^2f(\bm{\xi}_{t}))^{-1}\nabla_{\bm{\xi}}f(\bm{\xi}_{t}).
\end{align}

\begin{prop}
The optimal block mean approximation of $\mathbf{M}$ with the partition vector $\mathbf{s}$ according to the Frobenius norm 
\begin{align}
\min_{\bar{\mathbf{\Lambda}},\bar{\mathbf{B}}} \| \bar{\mathbf{\Lambda}}+\bar{\mathbf{B}} - \mathbf{M}  \|_F^2\;
\end{align}
is given by
\begin{align}
b_{ij} &= 
\begin{cases}
0,  &i = j, s_i = 1, \\
\frac{\sum_{mn} \mathbf{M}^{ii}_{mn} - \sum_{m} \mathbf{M}^{ii}_{mm}}{s_{i}(s_{i}-1)},  &i = j, s_i \neq 1, \\
\frac{\sum_{mn}\mathbf{M}^{ij}_{mn}}{s_{i}s_{j}}, &i \neq j, 
\end{cases} \label{b_ij} \\
\lambda_{i} &= \frac{1}{s_{i}}\sum_m \mathbf{M}^{ii}_{mm} - b_{ii}\;. \label{lambda_i}
\end{align}
\end{prop}
\begin{proof}
For  $i\neq j$, $\bar{\mathbf{\Lambda}}^{ij} = 0$ by construction. Hence, $b_{ij}=\frac{\sum_{mn}\mathbf{M}^{ij}_{mn}}{s_{i}s_{j}}$ is the minimum solution for  $\| \bar{\mathbf{B}}^{ij} - \mathbf{M}^{ij}  \|_F$. For $i=j$, since $b_{ii} = \frac{\sum_{mn} \mathbf{M}^{ii}_{mn} - \sum_{m} \mathbf{M}^{ii}_{mm}}{s_{i}(s_{i}-1)}$, the off-diagonal elements of $\mathbf{M}^{ii}$ are minimized under the Frobenius norm. And since $\lambda_{i} + b_{ii} = \frac{1}{s_{i}}\sum_m \mathbf{M}^{ii}_{mm}$, the diagonal elements of $\mathbf{M}^{ii}$ are also minimized.
\end{proof}

\begin{theorem}
For invertible matrix $\bar{\mathbf{\Lambda}}+\bar{\mathbf{B}}$, where $\bar{\mathbf{\Lambda}}$ is the diagonal expansion matrix of $\mathbf{\Lambda}$ and $\bar{\mathbf{B}}$ is the full expansion matrix of $\mathbf{B}$, both of which have the same partition vector $\mathbf{s}$, 
\begin{align}
(\bar{\mathbf{\Lambda}}+\bar{\mathbf{B}})^{-1} = \bar{\mathbf{\Lambda}}^{-1} + \bar{\mathbf{D}}
\end{align}
 where $\bar{\mathbf{D}}$ is the full expansion matrix with partition vector $\mathbf{s}$ of 
\begin{align}
\mathbf{D} =  (\mathbf{\Lambda} \mathbf{S} + \mathbf{S}\mathbf{B}\mathbf{S} )^{-1} - (\mathbf{\Lambda}\mathbf{S})^{-1} \label{D_theorem1}
\end{align}
where $\mathbf{S}= \text{\normalfont diag}(\mathbf{s})$.
\end{theorem}

\begin{proof}
First we prove $(\bar{\mathbf{\Lambda}}+\bar{\mathbf{B}})(\bar{\mathbf{\Lambda}}^{-1} + \bar{\mathbf{D}}) = \mathbf{I}$.
\begin{align}
\mathbf{D} &= (\mathbf{\Lambda}\mathbf{S} + \mathbf{S}\mathbf{B}\mathbf{S})^{-1} - (\mathbf{\Lambda}\mathbf{S})^{-1} \\
&= (\mathbf{\Lambda}\mathbf{S})^{-1} \nonumber \\
&- (\mathbf{\Lambda}\mathbf{S})^{-1} \mathbf{S}(\mathbf{I}+\mathbf{BS}(\mathbf{\Lambda}\mathbf{S})^{-1}\mathbf{S})^{-1}\mathbf{BS}(\mathbf{\Lambda}\mathbf{S})^{-1} \nonumber \\ 
&- (\mathbf{\Lambda}\mathbf{S})^{-1} \label{woodbury} \\
&= - \mathbf{\Lambda}^{-1}(\mathbf{I}+\mathbf{B}\mathbf{\Lambda}^{-1}\mathbf{S})^{-1}\mathbf{B}\mathbf{\Lambda}^{-1} \label{commutativity} \\
&= -(\mathbf{\Lambda}+\mathbf{B}\mathbf{S})^{-1}\mathbf{B}\mathbf{\Lambda}^{-1} \label{D_last}
\end{align}
(\ref{woodbury}) follows from the Kailath variant of Woodbury identity. (\ref{commutativity}) follows from $\mathbf{\Lambda}\mathbf{S} = \mathbf{S}\mathbf{\Lambda}$ since $\mathbf{S}$ and $\mathbf{\Lambda}$ are both diagonal. Multiply both sides of (\ref{D_last}) by $\mathbf{\Lambda}+\mathbf{B}\mathbf{S}$, after some manipulation, we have
\begin{align}
\mathbf{B}\mathbf{\Lambda}^{-1} + \mathbf{\Lambda}\mathbf{D} + \mathbf{B}\mathbf{S}\mathbf{D} = 0
\end{align}
Since $\bar{\mathbf{\Lambda}}$, $\bar{\mathbf{B}}$ and $\bar{\mathbf{D}}$ are the expansion matrices with partition vector $\mathbf{s}$ of $\mathbf{\Lambda}$, $\mathbf{B}$ and $\mathbf{D}$, respectively, we have equivalently
\begin{align}
\bar{\mathbf{B}}\bar{\mathbf{\Lambda}}^{-1} + \bar{\mathbf{\Lambda}}\bar{\mathbf{D}} + \bar{\mathbf{B}}\bar{\mathbf{D}} &= 0  \\
\mathbf{I} + \bar{\mathbf{B}}\bar{\mathbf{\Lambda}}^{-1} + \bar{\mathbf{\Lambda}}\bar{\mathbf{D}} + \bar{\mathbf{B}}\bar{\mathbf{D}} &= \mathbf{I}  \\
(\bar{\mathbf{\Lambda}}+\bar{\mathbf{B}})(\bar{\mathbf{\Lambda}}^{-1} + \bar{\mathbf{D}}) &= \mathbf{I}
\end{align}

Second, we prove $(\bar{\mathbf{\Lambda}}^{-1} + \bar{\mathbf{D}})(\bar{\mathbf{\Lambda}}+\bar{\mathbf{B}}) = \mathbf{I}$.
\begin{align}
\mathbf{D} &= (\mathbf{\Lambda}\mathbf{S} + \mathbf{S}\mathbf{B}\mathbf{S})^{-1} - (\mathbf{\Lambda}\mathbf{S})^{-1} \\
&= (\mathbf{\Lambda}\mathbf{S})^{-1}  \\
&- (\mathbf{\Lambda}\mathbf{S})^{-1} \mathbf{SB}(\mathbf{I}+\mathbf{S}(\mathbf{\Lambda}\mathbf{S})^{-1}\mathbf{SB})^{-1}\mathbf{S}(\mathbf{\Lambda}\mathbf{S})^{-1} \\ 
&- (\mathbf{\Lambda}\mathbf{S})^{-1} \\
&= - \mathbf{\Lambda}^{-1}\mathbf{B}(\mathbf{I}+\mathbf{\Lambda}^{-1}\mathbf{S}\mathbf{B})^{-1}\mathbf{\Lambda}^{-1} \\
&= - \mathbf{\Lambda}^{-1}\mathbf{B}(\mathbf{\Lambda}+\mathbf{S}\mathbf{B})^{-1}
\end{align}
Therefore
\begin{align}
\mathbf{\Lambda}^{-1}\mathbf{B} + \mathbf{D}\mathbf{\Lambda} + \mathbf{D}\mathbf{S}\mathbf{B} = 0  \\
\bar{\mathbf{\Lambda}}^{-1}\bar{\mathbf{B}} + \bar{\mathbf{D}}\bar{\mathbf{\Lambda}} + \bar{\mathbf{D}}\bar{\mathbf{B}} = 0 \\
\mathbf{I} + \bar{\mathbf{\Lambda}}^{-1}\bar{\mathbf{B}} + \bar{\mathbf{D}}\bar{\mathbf{\Lambda}} + \bar{\mathbf{D}}\bar{\mathbf{B}} = \mathbf{I} \\
(\bar{\mathbf{\Lambda}}^{-1} + \bar{\mathbf{D}})(\bar{\mathbf{\Lambda}}+\bar{\mathbf{B}}) = \mathbf{I}
\end{align}
\end{proof}

\begin{lemma}
For invertible matrix $\bar{\mathbf{\Lambda}}+\bar{\mathbf{B}}$, where $\bar{\mathbf{\Lambda}}$ is the diagonal expansion matrix of $\mathbf{\Lambda}$ and $\bar{\mathbf{B}}$ is the full expansion matrix of $\mathbf{B}$, both of which have the same partition vector $\mathbf{s}$,
\begin{align}
(\bar{\mathbf{\Lambda}}+\bar{\mathbf{B}})^{\frac{1}{2}} = \bar{\mathbf{\Lambda}}^{\frac{1}{2}} + \bar{\mathbf{D}}
\end{align}
 where $\bar{\mathbf{D}}$ is the full expansion matrix with partition vector $\mathbf{s}$ of 
\begin{align}
\mathbf{D} = \mathbf{S}^{-\frac{1}{2}}\left[(\mathbf{\Lambda} + \mathbf{S}^{\frac{1}{2}}\mathbf{B}\mathbf{S}^{\frac{1}{2}} )^{\frac{1}{2}} - \mathbf{\Lambda}^{\frac{1}{2}}\right]\mathbf{S}^{-\frac{1}{2}}  \label{D_lemma1}
\end{align}
where $\mathbf{S}= \text{\normalfont diag}(\mathbf{s})$.
\end{lemma}
\begin{proof}
\begin{align}
\mathbf{D} &= \mathbf{S}^{-\frac{1}{2}}\left[(\mathbf{\Lambda} + \mathbf{S}^{\frac{1}{2}}\mathbf{B}\mathbf{S}^{\frac{1}{2}} )^{\frac{1}{2}} - \mathbf{\Lambda}^{\frac{1}{2}}\right]\mathbf{S}^{-\frac{1}{2}}  
\end{align}
Left and right multiply both side by $\mathbf{S}^{\frac{1}{2}}$,
\begin{align}
\mathbf{S}^{\frac{1}{2}} \mathbf{D} \mathbf{S}^{\frac{1}{2}} &= (\mathbf{\Lambda} + \mathbf{S}^{\frac{1}{2}}\mathbf{B}\mathbf{S}^{\frac{1}{2}} )^{\frac{1}{2}} - \mathbf{\Lambda}^{\frac{1}{2}} \\
\mathbf{S}^{\frac{1}{2}} \mathbf{D} \mathbf{S}^{\frac{1}{2}}+ \mathbf{\Lambda}^{\frac{1}{2}} &= (\mathbf{\Lambda} + \mathbf{S}^{\frac{1}{2}}\mathbf{B}\mathbf{S}^{\frac{1}{2}} )^{\frac{1}{2}} \\
(\mathbf{S}^{\frac{1}{2}} \mathbf{D} \mathbf{S}^{\frac{1}{2}}+ \mathbf{\Lambda}^{\frac{1}{2}} )^{2} &= (\mathbf{\Lambda} + \mathbf{S}^{\frac{1}{2}}\mathbf{B}\mathbf{S}^{\frac{1}{2}} ) 
\end{align}
Expanding the square,
\begin{align}
&\mathbf{S}^{\frac{1}{2}} \mathbf{D} \mathbf{S} \mathbf{D} \mathbf{S}^{\frac{1}{2}} + \mathbf{S}^{\frac{1}{2}}\mathbf{D}\mathbf{S}^{\frac{1}{2}}\mathbf{\Lambda}^{\frac{1}{2}}
+ \mathbf{\Lambda}^{\frac{1}{2}}\mathbf{S}^{\frac{1}{2}}\mathbf{D}\mathbf{S}^{\frac{1}{2}}+\mathbf{\Lambda} \\
&= \mathbf{\Lambda} + \mathbf{S}^{\frac{1}{2}}\mathbf{B}\mathbf{S}^{\frac{1}{2}}
\end{align}
Left and right multiply both side by $\mathbf{S}^{-\frac{1}{2}}$, 
\begin{align}
\mathbf{D} \mathbf{S} \mathbf{D} + \mathbf{D}\mathbf{\Lambda}^{\frac{1}{2}}
+ \mathbf{\Lambda}^{\frac{1}{2}}\mathbf{D} 
= \mathbf{B}
\end{align}
Since $\bar{\mathbf{\Lambda}}$, $\bar{\mathbf{B}}$ and $\bar{\mathbf{D}}$ are the expansion matrices with partition vector $\mathbf{s}$ of $\mathbf{\Lambda}$, $\mathbf{B}$ and $\mathbf{D}$, respectively, we have equivalently
\begin{align}
\bar{\mathbf{D}} \bar{\mathbf{D}} + \bar{\mathbf{D}}\bar{\mathbf{\Lambda}}^{\frac{1}{2}}
+ \bar{\mathbf{\Lambda}}^{\frac{1}{2}}\bar{\mathbf{D}}
&= \bar{\mathbf{B}} \\
\bar{\mathbf{\Lambda}} + \bar{\mathbf{D}} \bar{\mathbf{D}} + \bar{\mathbf{D}}\bar{\mathbf{\Lambda}}^{\frac{1}{2}}
+ \bar{\mathbf{\Lambda}}^{\frac{1}{2}}\bar{\mathbf{D}} 
&= \bar{\mathbf{\Lambda}}  + \bar{\mathbf{B}}  \\
(\bar{\mathbf{\Lambda}}^{\frac{1}{2}}+\bar{\mathbf{D}})^2
&= \bar{\mathbf{\Lambda}} + \bar{\mathbf{B}}  \\
\bar{\mathbf{\Lambda}}^{\frac{1}{2}}+\bar{\mathbf{D}}
&= (\bar{\mathbf{\Lambda}} + \bar{\mathbf{B}})^{\frac{1}{2}}
\end{align}
\end{proof}

\begin{theorem}
For invertible matrix $\bar{\mathbf{\Lambda}}+\bar{\mathbf{B}}$, where $\bar{\mathbf{\Lambda}}$ is the diagonal expansion matrix of $\mathbf{\Lambda}$ and $\bar{\mathbf{B}}$ is the full expansion matrix of $\mathbf{B}$, both of which have the same partition vector $\mathbf{s}$,
\begin{align}
(\bar{\mathbf{\Lambda}}+\bar{\mathbf{B}})^{-\frac{1}{2}} = \bar{\mathbf{\Lambda}}^{-\frac{1}{2}} + \bar{\mathbf{D}}
\end{align}
 where $\bar{\mathbf{D}}$ is the full expansion matrix with partition vector $\mathbf{s}$ of 
\begin{align}
\mathbf{D} = \mathbf{S}^{-\frac{1}{2}}\left[(\mathbf{\Lambda} + \mathbf{S}^{\frac{1}{2}}\mathbf{B}\mathbf{S}^{\frac{1}{2}} )^{-\frac{1}{2}} - \mathbf{\Lambda}^{-\frac{1}{2}}\right]\mathbf{S}^{-\frac{1}{2}} 
\label{D_theorem2}
\end{align}
where $\mathbf{S}= \text{\normalfont diag}(\mathbf{s})$.
\end{theorem}
\begin{proof}
The theorem can be proved by combining results of Theorem 1 and Lemma 1.
Substituting $\mathbf{B}$ in (\ref{D_lemma1}) with $\mathbf{D}$ in (\ref{D_theorem1}) and substituting $\mathbf{\Lambda}$ in (\ref{D_lemma1}) with $\mathbf{\Lambda}^{-1}$, we get $\mathbf{D}$
\begin{align}
= &\mathbf{S}^{-\frac{1}{2}}[(\mathbf{\Lambda} + \mathbf{S}^{\frac{1}{2}} ((\mathbf{\Lambda} \mathbf{S} + \mathbf{S}\mathbf{B}\mathbf{S} )^{-1}  \\
-&(\mathbf{\Lambda}\mathbf{S})^{-1})\mathbf{S}^{\frac{1}{2}} )^{\frac{1}{2}} - \mathbf{\Lambda}^{-\frac{1}{2}}]\mathbf{S}^{-\frac{1}{2}}  \\
=&\mathbf{S}^{-\frac{1}{2}}[\mathbf{S}^{\frac{1}{2}} (\mathbf{\Lambda} \mathbf{S} + \mathbf{S}\mathbf{B}\mathbf{S} )^{-1} \mathbf{S}^{\frac{1}{2}} )^{\frac{1}{2}}- \mathbf{\Lambda}^{-\frac{1}{2}}]\mathbf{S}^{-\frac{1}{2}} \\
=&\mathbf{S}^{-\frac{1}{2}}[\mathbf{S}^{\frac{1}{2}} \mathbf{S}^{-\frac{1}{2}}(\mathbf{\Lambda}  + \mathbf{S}^{\frac{1}{2}}\mathbf{B}\mathbf{S}^{\frac{1}{2}} )^{-1}\mathbf{S}^{-\frac{1}{2}} \mathbf{S}^{\frac{1}{2}} )^{\frac{1}{2}}- \mathbf{\Lambda}^{-\frac{1}{2}}]\mathbf{S}^{-\frac{1}{2}} \\
=&  \mathbf{S}^{-\frac{1}{2}}\left[(\mathbf{\Lambda} + \mathbf{S}^{\frac{1}{2}}\mathbf{B}\mathbf{S}^{\frac{1}{2}} )^{-\frac{1}{2}} - \mathbf{\Lambda}^{-\frac{1}{2}}\right]\mathbf{S}^{-\frac{1}{2}} 
\end{align}

\end{proof}

\clearpage
\bibliography{bma}	

\begin{thebibliography}{29}
\providecommand{\natexlab}[1]{#1}
\providecommand{\url}[1]{\texttt{#1}}
\expandafter\ifx\csname urlstyle\endcsname\relax
  \providecommand{\doi}[1]{doi: #1}\else
  \providecommand{\doi}{doi: \begingroup \urlstyle{rm}\Url}\fi

\bibitem[Abu-El-Haija(2017)]{abu2017proportionate}
Abu-El-Haija, Sami.
\newblock Proportionate gradient updates with percentdelta.
\newblock \emph{arXiv}, 2017.

\bibitem[Amari(1998)]{amari1998natural}
Amari, Shun-ichi.
\newblock Natural gradient works efficiently in learning.
\newblock \emph{Neural Computation}, 1998.

\bibitem[Amari et~al.(1996)Amari, Cichocki, and Yang]{amari1996new}
Amari, Shun-ichi, Cichocki, Andrzej, and Yang, Howard~Hua.
\newblock A new learning algorithm for blind signal separation.
\newblock \emph{NIPS}, 1996.

\bibitem[Becker \& LeCun(1988)Becker and LeCun]{becker1988improving}
Becker, Sue and LeCun, Yann.
\newblock Improving the convergence of back-propagation learning with second
  order methods.
\newblock \emph{Technical Report}, 1988.

\bibitem[Botev et~al.(2017)Botev, Ritter, and Barber]{botev2017practical}
Botev, Aleksandar, Ritter, Hippolyt, and Barber, David.
\newblock Practical gauss-newton optimisation for deep learning.
\newblock \emph{arXiv}, 2017.

\bibitem[Chow \& Saad(1997)Chow and Saad]{chow1997approximate}
Chow, Edmond and Saad, Yousef.
\newblock Approximate inverse techniques for block-partitioned matrices.
\newblock \emph{SIAM Journal on Scientific Computing}, 1997.

\bibitem[Duchi et~al.(2011)Duchi, Hazan, and Singer]{duchi2011adaptive}
Duchi, John, Hazan, Elad, and Singer, Yoram.
\newblock Adaptive subgradient methods for online learning and stochastic
  optimization.
\newblock \emph{JMLR}, 2011.

\bibitem[Grosse \& Martens(2016)Grosse and Martens]{grosse2016kronecker}
Grosse, Roger and Martens, James.
\newblock A kronecker-factored approximate fisher matrix for convolution
  layers.
\newblock \emph{ICML}, 2016.

\bibitem[Grosse \& Salakhutdinov(2015)Grosse and
  Salakhutdinov]{grosse2015scaling}
Grosse, Roger and Salakhutdinov, Ruslan.
\newblock Scaling up natural gradient by factorizing fisher information.
\newblock \emph{ICML}, 2015.

\bibitem[Guillaume et~al.(2003)Guillaume, Huard, and
  Le~Calvez]{guillaume2003block}
Guillaume, Ph, Huard, A, and Le~Calvez, C.
\newblock A block constant approximate inverse for preconditioning large linear
  systems.
\newblock \emph{SIAM Journal on Matrix Analysis and Applications}, 2003.

\bibitem[Hazan et~al.(2007)Hazan, Agarwal, and Kale]{hazan2007logarithmic}
Hazan, Elad, Agarwal, Amit, and Kale, Satyen.
\newblock Logarithmic regret algorithms for online convex optimization.
\newblock \emph{Machine Learning}, 2007.

\bibitem[Honkela et~al.(2010)Honkela, Raiko, Kuusela, Tornio, and
  Karhunen]{honkela2010approximate}
Honkela, Antti, Raiko, Tapani, Kuusela, Mikael, Tornio, Matti, and Karhunen,
  Juha.
\newblock Approximate riemannian conjugate gradient learning for fixed-form
  variational bayes.
\newblock \emph{JMLR}, 2010.

\bibitem[Kingma \& Ba(2015)Kingma and Ba]{kingma2014adam}
Kingma, Diederik~P and Ba, Jimmy.
\newblock Adam: A method for stochastic optimization.
\newblock \emph{ICLR}, 2015.

\bibitem[Krummenacher et~al.(2016)Krummenacher, McWilliams, Kilcher, Buhmann,
  and Meinshausen]{krummenacher2016scalable}
Krummenacher, Gabriel, McWilliams, Brian, Kilcher, Yannic, Buhmann, Joachim~M,
  and Meinshausen, Nicolai.
\newblock Scalable adaptive stochastic optimization using random projections.
\newblock \emph{NIPS}, 2016.

\bibitem[Le~Roux et~al.(2008)Le~Roux, Manzagol, and
  Bengio]{roux2008topmoumoute}
Le~Roux, Nicolas, Manzagol, Pierre-Antoine, and Bengio, Yoshua.
\newblock Topmoumoute online natural gradient algorithm.
\newblock \emph{NIPS}, 2008.

\bibitem[Levenberg(1944)]{levenberg1944method}
Levenberg, Kenneth.
\newblock A method for the solution of certain non-linear problems in least
  squares.
\newblock \emph{Quarterly of Applied Mathematics}, 1944.

\bibitem[Marquardt(1963)]{marquardt1963algorithm}
Marquardt, Donald~W.
\newblock An algorithm for least-squares estimation of nonlinear parameters.
\newblock \emph{SIAM Journal on Applied Mathematics}, 1963.

\bibitem[Martens(2014)]{martens2014new}
Martens, James.
\newblock New insights and perspectives on the natural gradient method.
\newblock \emph{arXiv}, 2014.

\bibitem[Martens(2016)]{martens2016second}
Martens, James.
\newblock \emph{Second-order optimization for neural networks}.
\newblock PhD thesis, University of Toronto, 2016.

\bibitem[Martens \& Grosse(2015)Martens and Grosse]{martens2015optimizing}
Martens, James and Grosse, Roger.
\newblock Optimizing neural networks with kronecker-factored approximate
  curvature.
\newblock \emph{ICML}, 2015.

\bibitem[Nocedal \& Wright(2006)Nocedal and Wright]{nocedal2006numerical}
Nocedal, Jorge and Wright, Stephen~J.
\newblock Numerical optimization 2nd, 2006.

\bibitem[Ollivier(2015)]{ollivier2015riemannian}
Ollivier, Yann.
\newblock Riemannian metrics for neural networks i: feedforward networks.
\newblock \emph{arXiv}, 2015.

\bibitem[Pascanu \& Bengio(2014)Pascanu and Bengio]{pascanu2013revisiting}
Pascanu, Razvan and Bengio, Yoshua.
\newblock Revisiting natural gradient for deep networks.
\newblock \emph{ICLR}, 2014.

\bibitem[Peters \& Schaal(2008)Peters and Schaal]{peters2008natural}
Peters, Jan and Schaal, Stefan.
\newblock Natural actor-critic.
\newblock \emph{Neurocomputing}, 2008.

\bibitem[Shepherd(2012)]{shepherd2012second}
Shepherd, Adrian~J.
\newblock \emph{Second-order methods for neural networks: Fast and reliable
  training methods for multi-layer perceptrons}.
\newblock 2012.

\bibitem[Singh et~al.(2015)Singh, De, Zhang, Goldstein, and
  Taylor]{singh2015layer}
Singh, Bharat, De, Soham, Zhang, Yangmuzi, Goldstein, Thomas, and Taylor,
  Gavin.
\newblock Layer-specific adaptive learning rates for deep networks.
\newblock \emph{ICMLA}, 2015.

\bibitem[Tieleman \& Hinton(2012)Tieleman and Hinton]{hinton2012rmsprop}
Tieleman, Tijmen and Hinton, Geoffrey.
\newblock Rmsprop: Divide the gradient by a running average of its recent
  magnitude.
\newblock \emph{Neural networks for machine learning, Coursera}, 2012.

\bibitem[You et~al.(2017)You, Gitman, and Ginsburg]{you2017scaling}
You, Yang, Gitman, Igor, and Ginsburg, Boris.
\newblock Scaling sgd batch size to 32k for imagenet training.
\newblock \emph{arXiv}, 2017.

\bibitem[Zeiler(2012)]{zeiler2012adadelta}
Zeiler, Matthew~D.
\newblock Adadelta: an adaptive learning rate method.
\newblock \emph{arXiv}, 2012.

\end{thebibliography}
\bibliographystyle{icml2018}

\end{document}